\documentclass[preprint,12pt]{elsarticle}

\usepackage{xcolor}
\usepackage{enumitem}
\usepackage{tasks}
\usepackage{calc}

\usepackage[hidelinks]{hyperref}

\usepackage{multirow}

\usepackage{units}
\usepackage{amsmath}
\usepackage{amssymb}
\usepackage{mathtools}
\usepackage{amsthm}
\usepackage{stmaryrd}

\usepackage{dsfont}

\usepackage{algorithm}
\usepackage{algorithmic}

\usepackage{tikz}
\usepackage{pgfplots}
\usepgfplotslibrary{fillbetween}

\definecolor{myblue}{HTML}{1f77b4}
\definecolor{myorange}{HTML}{ff7f0e}

\newlength\height 
\newlength\width

\newcommand*{\VEC}[1]  {\ensuremath{\boldsymbol{#1}}}
\newcommand*{\MAT}[1]  {\ensuremath{\boldsymbol{#1}}}

\newcommand*{\nfeatures}{\ensuremath{p}}
\newcommand*{\nsamples}{\ensuremath{n}}
\newcommand*{\nmatrices}{\ensuremath{K}}
\newcommand*{\nclasses}{\ensuremath{Z}}
\newcommand*{\dof}{\ensuremath{\nu}}

\newcommand*{\ambient}{\ensuremath{\mathcal{E}}}
\newcommand*{\symspace}{\ensuremath{\mathcal{S}_{\nfeatures}}}
\newcommand*{\SPDman}{\ensuremath{\mathcal{S}^{++}_{\nfeatures}}}
\newcommand*{\SPDtangent}[1]{\ensuremath{T_{#1}\mathcal{S}^{++}_{\nfeatures}}}

\newcommand*{\metric}[3]{\ensuremath{\langle#2,#3\rangle_{#1}}}
\newcommand*{\LC}[2]{\ensuremath{\nabla_{#2}\, #1}}

\newcommand*{\eye}{\ensuremath{\MAT{I}_\nfeatures}}

\newcommand*{\covcenter}{\ensuremath{\MAT{G}}}
\newcommand*{\cov}{\ensuremath{\MAT{S}}}
\newcommand*{\data}{\ensuremath{\VEC{x}}}
\newcommand*{\dataMat}{\ensuremath{\MAT{X}}}

\newcommand*{\tangent}{\ensuremath{\MAT{\xi}}}
\newcommand*{\tangentBis}{\ensuremath{\MAT{\eta}}}

\newcommand*{\estimator}{\ensuremath{\MAT{\widehat{G}}}}

\newcommand*{\Normal}[1]{\ensuremath{\mathcal{N}(\VEC{0},#1)}}
\newcommand*{\Wishart}[1]{\ensuremath{\mathcal{W}(#1,\nsamples)}}
\newcommand*{\densitygenerator}{\ensuremath{h}}
\newcommand*{\EWishart}[1]{\ensuremath{\mathcal{EW}(#1,\nsamples,\densitygenerator)}}
\newcommand*{\tWishart}[1]{\ensuremath{t\textup{-}\mathcal{W}(#1,\nsamples,\dof)}}

\newtheorem{definition}{Definition}
\newtheorem{proposition}{Proposition}

\newtheorem{assumption}{Assumption}

\theoremstyle{definition}
\newtheorem{remark}{Remark}

\DeclareMathOperator{\tr}{tr}
\DeclareMathOperator{\diff}{d}
\DeclareMathOperator{\vvec}{vec}
\DeclareMathOperator{\symm}{sym}
\DeclareMathOperator{\expm}{expm}
\DeclareMathOperator{\logm}{logm}
\DeclareMathOperator{\grad}{grad}

\newcommand{\argmin}{\operatornamewithlimits{argmin}}
\newcommand{\argmax}{\operatornamewithlimits{argmax}}

\begin{document}

\begin{frontmatter}
    \title{Elliptical Wishart distributions: information geometry, maximum likelihood estimator, performance analysis and statistical learning}

    \author[label]{Imen Ayadi}
    \author[label]{Florent Bouchard}
    \author[label]{Frédéric Pascal}
    
    \affiliation[label]{
        organization={Université Paris-Saclay, CNRS, CentraleSupélec, laboratoire des signaux et systèmes},
        addressline={3 rue Joliot-Curie},
        city={Gif-sur-Yvette},
        postcode={91190},
        country={France}
    }

    \begin{abstract}
        This paper deals with Elliptical Wishart distributions -- which generalize the Wishart distribution -- in the context of signal processing and machine learning.
        Two algorithms to compute the maximum likelihood estimator (MLE) are proposed: a fixed point algorithm and a Riemannian optimization method based on the derived information geometry of Elliptical Wishart distributions.
        The existence and uniqueness of the MLE are characterized as well as the convergence of both estimation algorithms.
        Statistical properties of the MLE are also investigated such as consistency, asymptotic normality and an intrinsic version of Fisher efficiency.
        On the statistical learning side, novel classification and clustering methods are designed.
        For the $t$-Wishart distribution, the performance of the MLE and statistical learning algorithms are evaluated on both simulated and real EEG and hyperspectral data, showcasing the interest of our proposed methods.
        
    \end{abstract}
    
\end{frontmatter}

\section{Introduction}

Covariance matrices are tremendous in statistical signal processing and machine learning.
Indeed, they have shown very useful in various applications such as image processing~\cite{tuzel2008pedestrian}, electroencephalography (EEG)~\cite{barachant2011multiclass}, radar~\cite{pascal2008performance}, \textit{etc}.
In particular, they are exploited in the context of direction of arrival~\cite{mahot2013asymptotic}, change detection~\cite{prendes2015change,mian2024online}, source separation~\cite{pham2001blind,bouchard2018riemannian}, principal component analysis~\cite{jolliffe2016principal,collas2021probabilistic}, graph learning~\cite{friedman2008sparse,hippert2023learning}, \textit{etc}.
Covariance matrices have also been leveraged for classification and clustering.
In this context, statistics over covariance matrices are needed.
When dealing with covariances, the notorious Fréchet mean~\cite{bhatia2009positive} have extensively been employed~\cite{tuzel2008pedestrian,barachant2011multiclass,bouchard2024random}.
The Wishart distribution have also been exploited in~\cite{lee1999unsupervised}.

As of now, a large family of distributions over covariance matrices that appear quite appealing has been neglected: Elliptical Wishart distributions~\cite{teng1989generalized}.
These generalize the Wishart distribution the same way multivariate elliptical distributions generalize the Normal one%
\footnote{
    See, \textit{e.g.}, \cite{ollila2012complex} for a full review of multivariate elliptical distributions.
}.
In particular, as for the multivariate case, one can expect Elliptical Wishart distributions to yield robustness.
However, this robustness to noise and outliers is at the covariance level%
\footnote{
    In practical cases, this suits well mislabeling or large portions of corrupted data.
}.
Very little has been done to study Elliptical Wishart distributions.
From a statistical point of view, the latest advances to characterize them can be found in~\cite{ayadi2024elliptical}.
From a signal processing and machine learning perspective, our preliminary papers~\cite{ayadi2023elliptical,ayadi2023t} derive the information geometry, a Riemannian based algorithm to compute the maximum likelihood estimator, and a classification method%
\footnote{
    Their contributions are also included in the present paper for completeness.
}.
Notice that, while the considered model differs, the algorithm in~\cite{taylor2017generalization} provides a way to compute a center of covariance matrices that is related to the Elliptical Wishart maximum likelihood estimator%
\footnote{
    As a matter of fact, it can somehow be viewed as its Tyler estimator counterpart.
}.

The objectives of this paper are:
\textit{(i)} to provide practical tools to be able to exploit Elliptical Wishart distributions in signal processing and machine learning contexts;
\textit{(ii)} to theoretically study the properties of developed tools;
and \textit{(iii)} to validate them in practice.
Specifically, our contributions are:
\vspace*{-8pt}
\begin{itemize}[leftmargin=0pt, itemindent=*,itemsep=-4pt]
    \item Two algorithms to compute the maximum likelihood estimator: a fixed point algorithm and a Riemannian optimization method based on the information geometry of Elliptical Wishart distributions%
    \footnote{
        The proposed Riemannian algorithm is a little more sophisticated than the one in~\cite{ayadi2023elliptical}.
    }.
    The existence and uniqueness of the maximum likelihood estimator is characterized as well as the convergence of proposed algorithms.
    \item A theoretical performance analysis of the maximum likelihood estimator.
    Its expectation and variance are obtained.
    Asymptotic properties are also derived: consistency, asymptotic normality and an intrinsic version of Fisher efficiency.
    \item Statistical learning methods exploiting Elliptical Wishart distributions.
    A Bayesian classifier is designed by leveraging discriminant analysis.
    It generalizes the one from~\cite{ayadi2023t} to all Elliptical Wishart distributions.
    By exploiting $K$-means, the proposed Bayesian classifier is turned into a clustering method.
    \item An evaluation of the performance of the maximum likelihood estimator and statistical learning methods on both simulated and real EEG and hyperspectral data.
    It demonstrates the practical interest of Elliptical Wishart distributions.
\end{itemize}
\vspace*{-5pt}
Due to space limitations, all proofs are provided in Supplementary Materials.
To ensure reproducibility, the code related to this paper is available at \url{https://github.com/IA3005/Elliptical-Wishart-for-Signal-Processing.git}.

\section{Elliptical Wishart distributions}
Elliptical Wishart distributions form a large family of distributions over the manifold of symmetric positive definite matrices $\SPDman$~\cite{teng1989generalized,ayadi2024elliptical}.
They generalize the praised Wishart distribution.
In this section, we first properly define these distributions.
We then provide their information geometry, which will be used in the following to derive an algorithm for the MLE and study its properties.

\subsection{Statistical model}

Recall that the Wishart distribution basically corresponds to the distribution of (scaled) sample covariance matrices (SCM) of some i.i.d. random Gaussian vectors.
More specifically,  a random matrix $\cov\in\SPDman$ drawn from the Wishart distribution $\Wishart{\covcenter}$ with center $\covcenter\in\SPDman$ and $\nsamples$ degrees of freedom is $\cov=\dataMat\dataMat^\top$, where the columns $\{\data_i\}_{i=1}^\nsamples$ of $\dataMat\in\mathds{R}^{\nfeatures\times\nsamples}$ are $\nsamples$ i.i.d. random vectors drawn from the multivariate centered Normal distribution $\Normal{\covcenter}$.
In the Elliptical Wishart distributions extension, one still studies the distribution of (scaled) SCMs $\cov=\dataMat\dataMat^\top$.
However, $\dataMat\in\mathds{R}^{\nfeatures\times\nsamples}$ no longer has i.i.d. columns drawn from the multivariate Normal distribution.
Instead, $\dataMat$ is assumed to follow a so-called matrix-variate Elliptical distribution~\cite{teng1989generalized,ayadi2024elliptical}.
A formal definition of Elliptical Wishart distributions is provided in Definition~\ref{def:EW}.

\begin{definition}[Elliptical Wishart distributions~\cite{teng1989generalized}]
\label{def:EW}
    The Elliptical Wishart distribution $\EWishart{\covcenter}$, with center $\covcenter\in\SPDman$, degrees of freedom $\nsamples$ (with $\nsamples>\nfeatures$) and density generator $h:\mathds{R}^+\to\mathds{R}$, is the distribution whose random matrices $\cov\in\SPDman$ admit the probability density function (pdf)
    \begin{equation*}
        f(\cov|\covcenter) = \frac{\pi^{\nsamples\nfeatures/2}}{\Gamma_{\nfeatures}(\nsamples/2)}
        \det(\covcenter)^{-\nsamples/2} \det(\cov)^{(n-p-1)/2}
        \densitygenerator(\tr(\covcenter^{-1}\cov)),
    \end{equation*}
    where $\det(\cdot)$ and $\tr(\cdot)$ denote the determinant and trace operators, and $\Gamma_{\nfeatures}(\cdot)$ is the multivariate Gamma function of dimension $\nfeatures$.
    To ensure that $\densitygenerator(\cdot)$ is a proper density generator, one must have $\frac{\pi^{\nsamples\nfeatures/2}}{\Gamma(\nsamples\nfeatures/2)}\int_0^{+\infty}\densitygenerator(t)t^{\frac{\nsamples\nfeatures}{2}-1}dt = 1$.
\end{definition}

Two examples of Elliptical Wishart distributions, the Wishart and $t$-Wishart distributions are presented in Table~\ref{tab:EW_ex}.
Notice that Elliptical Wishart distributions generalize the Wishart distribution the same way multivariate Elliptical distributions generalize the multivariate Normal one.
Indeed, in both case, the exponential function in the pdf is replaced by a more general density generator $\densitygenerator(\cdot)$.
Hence, Elliptical Wishart distributions also allow to introduce robustness to noise and outliers.
However, unlike the multivariate case, it is not at the level of the covariance itself but at the level of its own distribution.
Meaning that one still deals with SCMs but some of these can be considered very noisy or even outliers%
\footnote{
    For instance, mislabeled data in the context of machine learning.
}.
Another striking similarity with the multivariate case is that random matrices drawn from some Elliptical Wishart distribution also admit a stochastic representation~\cite{ayadi2024elliptical}.
Given $\cov\sim\EWishart{\covcenter}$, one has
\begin{equation}
    \cov = \mathcal{Q} \, \covcenter^{1/2}\MAT{U}\MAT{U}^\top\covcenter^{1/2},
    \label{eq:stoch_rep}
\end{equation}
where $\mathcal{Q}$ and $\MAT{U}$ are two independent random variables.
$\MAT{U}$ is uniformly distributed on the unit sphere $\mathcal{C}_{\nfeatures,\nsamples}=\{\MAT{U}\in\mathds{R}^{\nfeatures\times\nsamples}: \|\MAT{U}\|_2=1\}$.
$\mathcal{Q}$ is a non-negative random scalar with pdf $t\mapsto\frac{\pi^{\nsamples\nfeatures/2}}{\Gamma_{\nfeatures}(\nsamples/2)}\densitygenerator(t)t^{\nsamples\nfeatures/2-1}$.


To study the information geometry, obtain the maximum likelihood estimator, \textit{etc.}, it remains to define the negative log-likelihood.
Given $\nmatrices$ samples $\{\cov_k\}_{k=1}^\nmatrices$, it is, up to an additive constant,
\begin{equation}
    L(\covcenter) =
    \frac{\nsamples\nmatrices}{2} \log\det(\covcenter)
    - \sum_{k=1}^\nmatrices \log\,\densitygenerator(\tr(\covcenter^{-1}\cov_k)).
    \label{eq:log_lik}
\end{equation}

\subsection{Fisher information geometry}
\label{subsec:info_geo}

The parameter $\covcenter$ of Elliptical Wishart distributions $\EWishart{\covcenter}$ belongs to the manifold of SPD matrices $\SPDman$.
A specific Riemannian geometry on $\SPDman$ is intrinsically linked to Elliptical Wishart distributions: their so-called Fisher information geometry.
Notice that this geometry was obtained in our previously published paper~\cite{ayadi2023elliptical}.
Because it is a conference paper, the description of the geometry is quite short.
For completeness, we provide full details in the present paper.
However, their is no additional contribution because the resulting Fisher information metric yields a very well known geometry on $\SPDman$.

The (smooth) manifold of SPD matrices $\SPDman$ is open in the (Euclidean) space of symmetric matrices $\symspace$.
Hence, the tangent space at every point $\covcenter$ in $\SPDman$ can be identified to $\symspace$, \textit{i.e.}, $\SPDtangent{\covcenter}\simeq\symspace$.
To characterize the information geometry, one first needs to determine the Fisher information metric.
At $\covcenter\in\SPDman$, given tangent vectors $\tangent$ and $\tangentBis\in\symspace$, it is obtained through~\cite{smith2005covariance}
\begin{equation}
    \metric{\covcenter}{\tangent}{\tangentBis} = \mathds{E}[\diff^2 L(\covcenter)[\tangent,\tangentBis]],
\end{equation}
where $\mathds{E}[\cdot]$ denotes the expectation, $\diff^2$ is the second order directional derivative and the negative log-likelihood $L$ is defined in~\eqref{eq:log_lik}.
The Fisher information metric of Elliptical Wishart distributions is provided in Proposition~\ref{prop:fim}.

\begin{proposition}[Fisher information metric]
    \label{prop:fim}
    The Fisher information metric of Elliptical Wishart distributions is, for $\covcenter\in\SPDman$, $\tangent$ and $\tangentBis\in\symspace$,
    \begin{equation*}
        \metric{\covcenter}{\tangent}{\tangentBis} =
        \alpha \tr(\covcenter^{-1}\tangent\covcenter^{-1}\tangentBis)
        + \beta \tr(\covcenter^{-1}\tangent)\tr(\covcenter^{-1}\tangentBis),
    \end{equation*}
    where, given $u(t)=-2h'(t)/h(t)$,
    \begin{equation*}
        \alpha = \frac{\nsamples}2\left(1+\frac{\mathds{E}[\mathcal{Q}^2u'(\mathcal{Q})]}{\nsamples\nfeatures(\frac{\nsamples\nfeatures}2+1)}\right)
        \qquad\quad
        \textup{and}
        \qquad\quad
        \beta = \frac{\nsamples}2(\alpha-\frac{\nsamples}{2}).
    \end{equation*}
\end{proposition}
\begin{proof}
    See supplementary materials.
\end{proof}

\begin{remark}
    From Cauchy-Schwarz inequality, the Fisher metric of Proposition~\ref{prop:fim} defines a proper Riemannian metric only if $\alpha>0$ and $\alpha+\nfeatures\beta>0$. 
    With a double integration by parts, we have $2\mathds{E}[\mathcal{Q}^2u'(\mathcal{Q})] =  \mathds{E}[\mathcal{Q}^2u(\mathcal{Q})^2]-\nsamples\nfeatures(\nsamples\nfeatures+2)$.
    Thus, the condition is equivalent to $\mathrm{var}[\mathcal{Q}u(\mathcal{Q})]>0$, which is fulfilled as long as $\mathcal{Q}u(\mathcal{Q})$ is not a constant almost surely.
    This happens to be true for every elliptical distribution.
\end{remark}

\begin{table}
    \centering
    \renewcommand{\arraystretch}{1.4}
    \begin{tabular}{ccccc}
        \hline
         distribution & $h(t)$ & $u(t)=-h'(t)/h(t)$  & $\alpha$ & $\beta$
         \\ \hline
         Wishart & $\exp(-t/2)$ & $1$ & $\nsamples/2$ & $0$
         \\
         $t$-Wishart & $(1+\frac{t}{\dof})^{-(\dof+\nsamples\nfeatures)/2}$ & $\frac{\dof+\nsamples\nfeatures}{\dof+t}$ & $\frac\nsamples2\frac{\dof+\nsamples\nfeatures}{\dof+\nsamples\nfeatures+2}$ & $\frac{-\nsamples^2}{2(\dof+\nsamples\nfeatures+2)}$
         \\[3pt] \hline
    \end{tabular}
    \caption{Functions and parameters that characterize the Wishart and $t$-Wishart distributions, from statistical modeling and information geometry point of views.
    In particular, see Definition~\ref{def:EW} and Proposition~\ref{prop:fim}.}
    \label{tab:EW_ex}
\end{table}

The parameters that correspond to two Elliptical Wishart distributions of particular interest in this paper, the Wishart and $t$-Wishart distributions are given in Table~\ref{tab:EW_ex}.
Interestingly, the Fisher information metric of Elliptical Wishart distributions is very similar to the one of multivariate Elliptical distributions~\cite{breloy2018intrinsic}.
The difference resides in the fact that $\alpha$ and $\beta$ depend on the degrees of freedom $\nsamples$ (number of samples) in the Elliptical Wishart case.
This means that Elliptical Wishart distributions and multivariate Elliptical distributions basically share the same information geometry on $\SPDman$.
In fact, this geometry is very well known: it corresponds to the most general form of the praised affine-invariant metric -- which is most known for $\alpha=1$ and $\beta=0$~\cite{bhatia2009positive}.
The geometry for this most general form can for instance be found in~\cite{breloy2018intrinsic}.
Let us recall it.

One of the most important tool in Riemannian geometry is the Levi-Civita connection.
Indeed, most other Riemannian objects are obtained from it, \textit{e.g.}, geodesics, parallel transport, Riemannian Hessian, \textit{etc}.
The Levi-Civita connection generalizes the notion of derivatives of vector fields%
\footnote{
    function that associates a unique tangent vector to every point on the manifold.
}
on manifolds.
On $\SPDman$ associated with the metric of Proposition~\ref{prop:fim}, it is defined, for $\covcenter\in\SPDman$, and vector fields $\tangent_{\covcenter}$ and $\tangentBis_{\covcenter}$ at $\covcenter$, as
\begin{equation}
    \LC{\tangentBis_{\covcenter}}{\tangent_{\covcenter}} = \diff \tangentBis_{\covcenter}[\tangent_{\covcenter}] - \symm(\tangent_{\covcenter}\covcenter^{-1}\tangentBis_{\covcenter}),
\end{equation}
where $\symm(\cdot)$ returns the symmetrical part of its argument.

From there, one can define geodesics, which generalize the concept of straight lines on manifold.
Indeed, these correspond to curves with no acceleration, \textit{i.e.}, $\gamma:\mathds{R}\to\SPDman$ such that $\LC{\dot{\gamma}(t)}{\dot{\gamma}(t)}=\MAT{0}$.
In our case, given a starting point $\covcenter$ and an initial direction $\tangent$, it is
\begin{equation}
    \gamma(t) = \covcenter\expm(t\covcenter^{-1}\tangent),
\end{equation}
where $\expm(\cdot)$ denotes the matrix exponential.
Geodesics allow to define the Riemannian exponential.
For $\covcenter$, it is the mapping $\exp_{\covcenter}:\symspace\to\SPDman$ such that for all $\tangent$, $\exp_{\covcenter}(\tangent)=\gamma(1)$, where $\gamma$ is the geodesic with starting point $\covcenter$ and initial direction $\tangent$, \textit{i.e.},
\begin{equation}
    \exp_{\covcenter}(\tangent) = \covcenter\expm(\covcenter^{-1}\tangent).
    \label{eq:exp}
\end{equation}
In practice, computing the Riemannian exponential can be quite demanding (computationally) and it might be advantageous to limit ourselves to an approximation of it, especially in the context of optimization.
Such a mapping from the tangent spaces onto the manifold is called a retraction~\cite{absil2008optimization}.
In the case of $\SPDman$, one of the best solution, is the second-order approximation given by~\cite{jeuris2012survey}
\begin{equation}
    R_{\covcenter}(\tangent) = \covcenter + \tangent + \frac12 \tangent \covcenter^{-1} \tangent.
    \label{eq:retr}
\end{equation}
From the Riemannian exponential, one can define the Riemannian logarithm, which is the inverse of the Riemannian exponential.
Given $\covcenter$, it is the mapping $\log_{\covcenter}:\symspace\to\SPDman$ such that, for $\cov$,
\begin{equation}
    \log_{\covcenter}(\cov) = \covcenter\logm(\covcenter^{-1}\cov),
\end{equation}
where $\logm(\cdot)$ denotes the matrix logarithm.
These allow to obtain the Fisher-Rao distance for Elliptical Wishart distributions on $\SPDman$.
Given $\covcenter$ and $\cov$, it is given by~\cite{breloy2018intrinsic}
\begin{equation}
    \delta^2(\covcenter,\cov) = \alpha \| \logm(\covcenter^{-1/2}\cov\covcenter^{-1/2}) \|_2^2 + \beta (\log\det(\covcenter^{-1}\cov))^2.
    \label{eq:dist}
\end{equation}

The last geometrical object that we introduce in the present paper is parallel transport.
It is also a very important object of Riemannian geometry, especially when it comes to optimization as it for instance enables to generalize the classical conjugate gradient and BFGS algorithms.
Indeed, it allows to transport tangent vectors from the tangent space of one point onto the tangent space of another point.
Again, it is defined thanks to the Levi-Civita connection.
Given some curve $\gamma:\mathds{R}\to\SPDman$, the idea is to transport a tangent vector $\tau(0)$ at $\gamma(0)$ along $\gamma(\cdot)$ so that one gets the corresponding tangent vector $\tau(t)$ at $\gamma(t)$ for all $t$.
It is ``parallel'' because it is solution to $\LC{\tau(t)}{\dot{\gamma}(t)} = \MAT{0}$.
In our case, given the geodesic with starting point $\covcenter$ and direction $\tangent$, transporting the tangent vector $\tangentBis$ is given by~\cite{jeuris2012survey}
\begin{equation}
    \tau(t) = \expm(t\tangent\covcenter^{-1}/2) \tangentBis \expm(t\covcenter^{-1}\tangent/2).
\end{equation}
Given $\covcenter$ and $\cov$ in $\SPDman$, it yields the following vector transport
\begin{equation}
    \mathcal{T}_{\covcenter\rightarrow\cov}(\tangentBis) = (\cov\covcenter^{-1})^{1/2} \tangentBis (\covcenter^{-1}\cov)^{1/2},
    \label{eq:vec_transp}
\end{equation}
which transports $\tangentBis$ from $\covcenter$ to $\cov$.

\section{Maximum likelihood estimator and convergence analysis}
\label{sec:mle}

In this section, the maximum likelihood estimator (MLE) is computed in two different ways: \textit{(i)} a fixed point algorithm; and \textit{(ii)} an algorithm exploiting Riemannian optimization with the information geometry of Section~\ref{subsec:info_geo}.
While the fixed point algorithm was not previously published, a simpler version of the Riemannian optimization one can be found in~\cite{ayadi2023elliptical}.
However, the one in this paper is a little more sophisticated since it is not limited to steepest descent and allows for the use of conjugate gradient or BFGS algorithms.
The two other major contributions of this section, are: \textit{(i)} the existence and uniqueness of the maximum likelihood estimator; and \textit{(ii)} the convergence analysis of proposed algorithms.

\subsection{Maximum likelihood estimator}
\label{subsec:mle:algos}

Given $\nmatrices$ samples $\{\cov_k\}_{k=1}^{\nmatrices}$, the maximum likelihood estimator $\estimator$ of an Elliptical Wishart distribution $\EWishart{\covcenter}$ is defined as the solution to the optimization problem
\begin{equation}
    \argmin_{\covcenter\in\SPDman} \quad L(\covcenter),
    \label{eq:mle}
\end{equation}
where $L(\cdot)$ is the negative log-likelihood defined in~\eqref{eq:log_lik}.
In the general case, no closed form solution can be obtained to solve this problem%
\footnote{
    To our knowledge a closed form solution is only available for the Wishart distribution.
}.
To derive iterative procedures that yield the maximum likelihood estimator, the first thing is to do is to compute the Euclidean gradient of $L(\cdot)$.
This is achieved in Proposition~\ref{prop:eucl_grad}.

\begin{proposition}[Euclidean gradient of the negative log-likelihood~\eqref{eq:log_lik}]
    \label{prop:eucl_grad}
    The Euclidean gradient $\grad_{\ambient} L(\covcenter)$ of the negative log-likelihood $L(\cdot)$ at $\covcenter\in\SPDman$ is given by
    \begin{equation*}
        \grad_{\ambient} L(\covcenter) = \frac12 \covcenter^{-1} \left( \nsamples\nmatrices\covcenter - \sum_{k=1}^\nmatrices u(\tr(\covcenter^{-1}\cov_k)) \cov_k \right) \covcenter^{-1},
    \end{equation*}
    where $u(\cdot)$ is defined in Proposition~\ref{prop:fim}. 
\end{proposition}
\begin{proof}
    See supplementary materials.
\end{proof}
%

As for the multivariate case, the fixed-point algorithm arises by setting the equation characterizing critical points, \textit{i.e.}, $\grad_{\ambient} L(\covcenter)=\MAT{0}$.
Indeed, it yields the following fixed-point equation
\begin{equation}
    \covcenter = \frac1{\nsamples\nmatrices}\sum_{k=1}^\nmatrices u(\tr(\covcenter^{-1}\cov_k)) \cov_k.
    \label{eq:fixed_point}
\end{equation}
Recall that examples of the function $u(\cdot)$ for Wishart and $t$-Wishart distributions are given in Table~\ref{tab:EW_ex}. 
Notice that in the Wishart case, from $u(t)=1$, we immediately get the closed form solution
\begin{equation}
    \estimator_{\textup{W}} = \frac1{\nsamples\nmatrices}\sum_{k=1}^\nmatrices \cov_k.
    \label{eq:mle_Wishart}
\end{equation}

Equation~\eqref{eq:fixed_point} naturally yields the fixed-point algorithm provided in Algorithm~\ref{algo:fixed_point}.
This algorithm is very similar to the one obtained in the multivariate elliptical case, or more generally in the case of $M$-estimators~\cite{maronna1976robust,tyler1987distribution}.
Even though it is quite simple, these algorithms have shown themselves very powerful in the latter cases.
Notice that, contrary to the multivariate case for which one needs at least $\nsamples=\nfeatures+1$ samples for the algorithm to be defined, there is no condition on the number of matrices $\nmatrices$ here, \textit{i.e.}, it works even for $\nmatrices=1$.
This is simply due to the fact that matrices $\cov_k$ are SPD.

\begin{algorithm}[t!]
    \caption{Elliptical Wishart MLE -- fixed-point algorithm}
    \label{algo:fixed_point}
    \begin{algorithmic}
        \STATE {\bfseries Input:} matrices $\{\cov_k\}_{k=1}^\nmatrices$, degrees of freedom $\nsamples$, function $u(\cdot)$, initial guess $\covcenter_0$.
        \STATE $t=0$
        \REPEAT
            \STATE $\covcenter_{t+1} = \frac1{\nsamples\nmatrices}\sum_{k=1}^\nmatrices u(\tr(\covcenter_t^{-1}\cov_k)) \cov_k$
            \STATE $t = t+1$
        \UNTIL{convergence \textbf{or} maximum iterations}
        \STATE {\bfseries Return:} maximum likelihood estimator $\estimator=\covcenter_t$
    \end{algorithmic} 
\end{algorithm}

To obtain a Riemannian optimization based algorithm, the first step is to compute the Riemannian gradient.
This is achieved by transforming the Euclidean gradient derived in Proposition~\ref{prop:eucl_grad}.
To do so, one exploits Proposition~\ref{prop:egrad2rgrad}, which allows to obtain the Riemannian gradient of $L(\cdot)$ on $\SPDman$ equipped with the Fisher metric from Proposition~\ref{prop:fim} from the Euclidean gradient.
This formula was already derived in~\cite{bouchard2021riemannian}.
\begin{proposition}[Euclidean gradient to Riemannian gradient~\cite{bouchard2021riemannian}]
    \label{prop:egrad2rgrad}
    Given a cost function $L:\SPDman\to\mathds{R}$ with Euclidean gradient $\grad_{\ambient} L(\covcenter)$ at $\covcenter\in\SPDman$, one obtains the Riemannian gradient $\grad L(\covcenter)$ on $\SPDman$ equipped with the Fisher metric of Proposition~\ref{prop:fim} through the formula
    \begin{equation*}
        \grad L(\covcenter) = \frac1\alpha \covcenter \grad_{\ambient} L(\covcenter) \covcenter - \frac{\beta}{\alpha(\alpha+\nfeatures\beta)}\tr(\grad_{\ambient} L(\covcenter) \covcenter) \covcenter.
    \end{equation*}
\end{proposition}
%

Given the iterate $\covcenter_t$, the goal is to obtain a descent direction of $L(\cdot)$ from the Riemannian gradient $\grad L(\covcenter_t)$, .
To achieve this, several possibilities exist; see, \textit{e.g.},~\cite{absil2008optimization,boumal2023introduction}.
The simplest one is the steepest descent, which is
\begin{equation}
    \tangent_t = - \grad L(\covcenter_t),
    \label{eq:steep_desc}
\end{equation}
Exploiting the vector transport~\eqref{eq:vec_transp}, more sophisticated options are available, such as descent directions from conjugate gradient or BFGS algorithms.
For instance, the descent direction of the conjugate gradient procedure is
\begin{equation}
    \tangent_t =  - \grad L(\covcenter_t) + \kappa_t \mathcal{T}_{\covcenter_{t-1}\to\covcenter_t}(\lambda_{t-1}\tangent_{t-1}),
    \label{eq:conj_grad}
\end{equation}
where $\lambda_{t-1}$ is the stepsize of the previous iteration, and $\kappa_t$ is a scalar that can be computed with a rule as the ones in~\cite{absil2008optimization}.
The next iterate is then
\begin{equation}
    \covcenter_{t+1} = R_{\covcenter_t}(\lambda_t\tangent_t),
\end{equation}
where $\lambda_t$ is the stepsize, which can for instance be computed with a linesearch, and $R(\cdot)$ is the retraction defined in~\eqref{eq:retr}.
Alternatives to~\eqref{eq:retr} can be considered, such as the exponential mapping~\eqref{eq:exp}.
The Riemannian optimization based algorithm to estimate the Elliptical Wishart maximum likelihood estimator is summarized in Algorithm~\ref{algo:riem}.

\begin{algorithm}
    \caption{Elliptical Wishart MLE -- Riemannian optimization algorithm}
    \label{algo:riem}
    \begin{algorithmic}
        \STATE {\bfseries Input:} matrices $\{\cov_k\}_{k=1}^\nmatrices$, degrees of freedom $\nsamples$, function $u(\cdot)$, initial guess $\covcenter_0$.
        \STATE $t=0$
        \REPEAT
            \STATE Compute $\grad L(\covcenter_t)$ with Propositions~\ref{prop:eucl_grad} and~\ref{prop:egrad2rgrad}
            \STATE Compute descent direction $\tangent_t$ with, \textit{e.g.},~\eqref{eq:steep_desc} or~\eqref{eq:conj_grad}
            \STATE Compute stepsize $\lambda_t$ and next iterate $\covcenter_{t+1} = R_{\covcenter_t}(\lambda_t\tangent_t)$
            \STATE t=t+1
        \UNTIL{convergence \textbf{or} maximum iterations}
        \STATE {\bfseries Return:} maximum likelihood estimator $\estimator=\covcenter_t$
    \end{algorithmic}
\end{algorithm}

\begin{remark}
    The fixed point algorithm from Algorithm~\ref{algo:fixed_point} can be seen as a special case of the Riemannian optimization based algorithm in Algorithm~\ref{algo:riem}.
    Indeed, the fixed point algorithm corresponds to the Riemannian one with (i) the steepest descent procedure~\eqref{eq:steep_desc} for the Riemannian gradient associated to parameters $\alpha=1$ and $\beta=0$ of the Fisher metric; (ii) constant stepsize $\lambda_t=1/\nsamples\nmatrices$; and (iii) the Euclidean retraction, i.e., $R_{\covcenter}(\tangent)=\covcenter+\tangent$.
\end{remark}

We now have two iterative algorithms to compute the Elliptical Wishart maximum likelihood estimator in practice.
Even though one is a special case of the other, we still chose to differentiate them due to the historical importance of fixed-point algorithms in the elliptical literature.
Interestingly, as we will see later, contrary to the multivariate case, the Riemannian optimization based algorithm is significantly more efficient (in terms of iterations before convergence) than the fixed-point one.
To ensure that these algorithms actually provide meaningful estimators, it remains to study the optimization problem and to conduct a convergence analysis of the algorithms.

\subsection{Existence and uniqueness of the maximum likelihood estimator}
\label{subsec:mle:exist_unique}

The first thing to do when it comes to check that Algorithms are actually meaningful in practice is to verify that the maximum likelihood estimator resulting from the optimization problem~\eqref{eq:mle} exists and is unique.
This is the objective of this subsection.
This is achieved by studying the fixed-point equation~\eqref{eq:fixed_point}.
To obtain existence and uniqueness, we also need additional assumptions.
To state them, let us introduce the function $\psi:\mathds{R}^{+}\to\mathds{R}$ such that, for all $t\in\mathds{R}^{+}$, $\psi(t)=t u(t)$.
These additional assumptions are given in Assumptions~\ref{assump:exist_unique_1}, \ref{assump:exist_unique_2} and~\ref{assump:exist_unique_3}.

\begin{assumption}
    \label{assump:exist_unique_1}
    $u:\mathds{R}^{+}\to\mathds{R}$ defined in Proposition~\ref{prop:fim} is continuous, non-increasing and non-negative -- strictly non-negative for $t>0$.
\end{assumption}

\begin{assumption}
    \label{assump:exist_unique_2}
    $\psi:\mathds{R}^{+}\to\mathds{R}$ is non-decreasing, and its supremum is greater than $\nsamples\nfeatures$, \textit{i.e.}, $\sup_{t\geq0} \psi(t) > \nsamples\nfeatures$.
    Notice that, $\sup_{t\geq0} \psi(t)=+\infty$ is allowed.
\end{assumption}

\begin{assumption}
    \label{assump:exist_unique_3}
    $\psi:\mathds{R}^{+}\to\mathds{R}$ is strictly increasing when $\psi(t)<\sup_{t\geq0} \psi(t)$.
\end{assumption}

\begin{remark}
    The assumptions required to prove the existence of $M$-estimators -- which include maximum likelihood estimators of multivariate Elliptical distributions -- are somewhat similar to Assumptions~\ref{assump:exist_unique_1}, \ref{assump:exist_unique_2} and~\ref{assump:exist_unique_3}~\cite{maronna1976robust}.
    The main difference is that $M$-estimators require an additional assumption specifying that data should not be too concentrated in low-dimensional linear subspaces~\cite[Condition E]{maronna1976robust}.
    This is not needed in our case because we are dealing with SPD samples.
\end{remark}

In practice, it is possible to verify that many Elliptical Wishart distributions fulfill Assumptions~\ref{assump:exist_unique_1}, \ref{assump:exist_unique_2} and~\ref{assump:exist_unique_3}.
In particular, let us look at our two examples, the Wishart and $t$-Wishart distributions, for which the corresponding functions $u(\cdot)$ are given in Table~\ref{tab:EW_ex}.
For Wishart, since $u(t)=1$, the three assumptions are readily checked.
Since in that case, the solution is known in closed form -- see~\eqref{eq:mle_Wishart} -- this has no practical use.
For the $t$-Wishart distribution, from $u(t)=\frac{\dof+\nsamples\nfeatures}{\dof+t}$, one can easily see that, for $t\geq0$, $t\mapsto u(t)$ is indeed continuous, non-increasing and strictly positive.
Moreover, $t\mapsto\psi(t)$ is also strictly increasing and bounded, with $\sup_{t\geq0} \psi(t) = \nsamples\nfeatures+\dof > \nsamples\nfeatures$.
This shows that all three assumptions are verified in this case.

Assumptions~\ref{assump:exist_unique_1}, \ref{assump:exist_unique_2} and~\ref{assump:exist_unique_3} are enough to ensure existence and uniqueness of a solution to the fixed-point equation~\eqref{eq:fixed_point}, which characterize the maximum likelihood estimator solutions.
This is proved in Proposition~\ref{prop:exist_unique}.

\begin{proposition}[Existence and uniqueness of Elliptical Wishart MLE]
    \label{prop:exist_unique}
    Under Assumptions~\ref{assump:exist_unique_1} and~\ref{assump:exist_unique_2}, equation~\eqref{eq:fixed_point} admits a solution on $\SPDman$.
    Assumption~\ref{assump:exist_unique_3} further ensures this solution to be unique. 
\end{proposition}
\begin{proof}
    See supplementary materials.
\end{proof}

\subsection{Convergence analysis of fixed-point and Riemannian based algorithms}

It is now established that the maximum likelihood estimator of Elliptical Wishart distributions -- defined through the optimization problem~\eqref{eq:mle} -- exists and is unique.
The next step is to study the convergence of Algorithms~\ref{algo:fixed_point} and~\ref{algo:riem}, which aim to estimate this maximum likelihood estimator.
As explained in Section~\ref{subsec:mle:algos}, Algorithm~\ref{algo:fixed_point} can be viewed as a special case of Algorithm~\ref{algo:riem}.
However, we still study their convergence separately by leveraging very different tools.
Concerning the fixed-point algorithm given in Algorithm~\ref{algo:fixed_point}, its convergence is straightforward to determine.
It is provided in Proposition~\ref{prop:convergence_fixed_point}.

\begin{proposition}[Convergence of the fixed-point algorithm]
    \label{prop:convergence_fixed_point}
    Given that Assumptions~\ref{assump:exist_unique_1}, \ref{assump:exist_unique_2} and~\ref{assump:exist_unique_3} are satisfied -- which guarantee the existence and uniqueness of the maximum likelihood estimator -- the sequence $\{\covcenter_t\}$, produced by Algorithm~\ref{algo:fixed_point} from any starting point $\covcenter_0\in\SPDman$, converges to the maximum likelihood estimator $\estimator$ defined in~\eqref{eq:mle}.
\end{proposition}
\begin{proof}
    See supplementary materials.
\end{proof}

We can now study the convergence of the Riemannian optimization based algorithm to obtain the maximum likelihood estimator~\eqref{eq:mle} described in Algorithm~\ref{algo:riem}.
To achieve this, we leverage the concept of (strict) geodesic convexity, which generalizes usual (strict) convexity to Riemannian manifolds.
Formally, an objective function $L:\SPDman\to\mathds{R}$ is said geodesically convex if, for $\covcenter$, $\MAT{H}\in\SPDman$,
\begin{equation}
    \forall t \in ]0,1[,
    \qquad
    L(\gamma(t)) \leq (1-t) L(\covcenter) + t L(\MAT{H}),
\end{equation}
where $\gamma(\cdot)$ is the geodesic joining $\covcenter$ and $\MAT{H}$.
Furthermore, $L(\cdot)$ is strictly geodesically convex if the inequality is strict for $\covcenter\neq\MAT{H}$.
Geodesic convexity and strict geodesic convexity have several very interesting consequences.
If $L(\cdot)$ is geodesically convex, then each local minimum is a global minimizer on $\SPDman$.
If $L(\cdot)$ is further strictly geodesically convex, the uniqueness of the global minimizer is ensured (if it exists).
Strict geodesic convexity also yields that there exists a neighborhood $\mathcal{U}\subset\SPDman$ of the solution $\estimator$ such that, for every initialization $\covcenter_0\in\mathcal{U}$, Algorithm~\ref{algo:riem} converges to the solution.
Indeed, strict geodesic convexity implies that the Hessian is positive definite, which ensures convergence in a neighborhood~\cite[Theorem 6.3.2]{absil2008optimization}.

To ensure that the negative log-likelihood $L:\SPDman\to\mathds{R}$ defined in~\eqref{eq:log_lik} is geodesically convex -- respectively strictly geodesically convex --, Assumption~\ref{assump:geo_convex} -- respectively Assumption~\ref{assump:strict_geo_convex} -- is required.
Notice that these assumptions, which concern the density generator, are similar to the ones in the multivariate case~\cite{wiesel2011unified,ollila2014regularized}.
From there, geodesic convexity -- respectively strict geodesic convexity -- is established in Proposition~\ref{prop:geo_convex}.
%

\begin{assumption}
    \label{assump:geo_convex}
    $t\mapsto-\log\densitygenerator(t)$ is non-decreasing and $s\mapsto-\log\densitygenerator(e^s)$ is convex.
\end{assumption}

\begin{assumption}
    \label{assump:strict_geo_convex}
    $t\mapsto-\log\densitygenerator(t)$ is strictly increasing and $s\mapsto-\log\densitygenerator(e^s)$ is strictly convex.
\end{assumption}

\begin{remark}
    Notice that the condition $t\mapsto-\log\densitygenerator(t)$ non-decreasing is equivalent to $u:\mathds{R}^+\to\mathds{R}$ non-negative, which is included in Assumption~\ref{assump:exist_unique_1}.
    Furthermore, assuming that $u(\cdot)$ is differentiable, $s\mapsto-\log\densitygenerator(e^s)$ convex is equivalent to $\psi(t) = u(t) + tu'(t) \geq 0$ for $t>0$, \textit{i.e.}, $\psi$ non-decreasing, which is part of Assumption~\ref{assump:exist_unique_2}.
\end{remark}

\begin{proposition}[Geodesic convexity of the negative log-likelihood~\eqref{eq:log_lik}]
    \label{prop:geo_convex}
    Under Assumption~\ref{assump:geo_convex}, the negative log-likelihood~\eqref{eq:log_lik} of Elliptical Wishart distributions is geodesically convex.
    If Assumption~\ref{assump:strict_geo_convex} is satisfied, the negative log-likelihood~\eqref{eq:log_lik} is further strictly geodesically convex.
\end{proposition}
\begin{proof}
    See supplementary materials.
\end{proof}

\section{Performance analysis}
From Section~\ref{sec:mle}, we have obtained the maximum likelihood estimator of Elliptical Wishart distributions, \textit{i.e.}, two algorithms to compute it are derived, its existence and uniqueness are proven, and the convergence of proposed algorithms is determined.
To better characterize this maximum likelihood estimator, it remains to analyze its performance from a statistical point of view.
This is the objective of the present section.
Specifically, we derive its expectation and variance, its consistency, and its asymptotic normality.
We further obtain a new property closely related to the so-called intrinsic Cramér-Rao bound~\cite{smith2005covariance,boumal2013intrinsic,breloy2018intrinsic} that we call intrinsic Fisher efficiency.
In this section, the maximum likelihood estimator of $K$ samples $\{\cov_k\}_{k=1}^\nmatrices$ is denoted $\estimator_{\nmatrices}$.
Throughout the section, it is assumed that Assumptions~\ref{assump:exist_unique_1}, \ref{assump:exist_unique_2} and~\ref{assump:exist_unique_3} ensuring the existence and uniqueness of this maximum likelihood estimator hold.

Our first task is, given $\nmatrices\in\mathds{N}$, to find the expectation and variance of the maximum likelihood estimator $\estimator_{\nmatrices}$ with respect to all possible realizations of the $\nmatrices$ samples $\{\cov_k\}_{k=1}^\nmatrices$.
To achieve this, one must first verify that the maximum likelihood estimator possesses the invariance property of the Elliptical Wishart model with respect to affine transformations.
This is done in Proposition~\ref{prop:mle_invariant}.
The expectation and variance of the maximum likelihood estimator $\estimator_{\nmatrices}$ are then provided in Proposition~\ref{prop:mle_expec_var}.

\begin{proposition}[invariance of $\estimator_{\nmatrices}$ w.r.t. affine transformations]
    \label{prop:mle_invariant}
    For any non-singular matrix $\MAT{A}\in\mathds{R}^{\nfeatures\times\nfeatures}$, if $\estimator_{\nmatrices}$ is the maximum likelihood estimator of $\{\cov_k\}_{k=1}^\nmatrices$, then the one of $\{\MAT{A}\cov_k\MAT{A}^\top\}_{k=1}^\nmatrices$ is $\MAT{A}\estimator_{\nmatrices}\MAT{A}^\top$.
\end{proposition}
\begin{proof}
    This is a direct consequence of the fixed-point equation~\eqref{eq:fixed_point} and the uniqueness of the maximum likelihood estimator.
\end{proof}

\begin{proposition}[Expectation and variance of $\estimator_{\nmatrices}$]
    \label{prop:mle_expec_var}
    There exists $\mu>0$, $\sigma_1>0$ and $\sigma_2\geq-2\sigma_1/\nfeatures$ such that
    \begin{equation*}
        \begin{array}{rcl}
             \mathds{E}[\estimator_\nmatrices] & = & \mu\covcenter \\
             \mathrm{var}(\estimator_\nmatrices) & = & \sigma_1 (\eye+\MAT{K}_{\nfeatures\nfeatures})(\covcenter\otimes\covcenter) + \sigma_2\vvec(\covcenter)\vvec(\covcenter)^\top,
        \end{array}
    \end{equation*}
    where $\MAT{K}_{\nfeatures\nfeatures}$ denotes the commutation matrix.
\end{proposition}
\begin{proof}
    See supplementary materials.
    Notice that this proof exploits~\cite[Proposition 13.6]{bilodeau1999theory}.
\end{proof}
%

Notice that for the Wishart distribution, it can be shown that the Wishart estimator $\estimator_{\textup{W}}$ given by~\eqref{eq:mle_Wishart} follows the Wishart distribution $\mathcal{W}(\frac1{\nsamples\nmatrices}\covcenter, \nsamples\nmatrices)$.
Indeed, this stems from the fact that if $\cov_k\sim\Wishart{\covcenter}$, then $\sum_{k=1}^\nmatrices\cov_k\sim\mathcal{W}(\nsamples\nmatrices,\covcenter)$~\cite{bilodeau1999theory}.
The expectation and variance of the Wishart distribution are well known~\cite{bilodeau1999theory}.
Hence, $\estimator_{\textup{W}}$ is unbiased, \textit{i.e.}, $\mu=1$, and its variance is obtained for $\sigma_1=\frac{1}{\nsamples\nmatrices}$ and $\sigma_2=0$.

Now, our focus is to derive asymptotic properties of the maximum likelihood estimator $\estimator_\nmatrices$ as the number of matrices $\nmatrices$ grows toward infinity.
In the present work, we first investigate two crucial properties when conducting the theoretical study of an estimator: consistency and asymptotic normality.
Consistency is indeed tremendous as it ensures that when the number of available samples goes to infinity, the estimator actually tends to the true parameter.
Asymptotic normality is also very important as it shows that the estimator asymptotically follows a Normal distribution as $\nmatrices\to+\infty$.
The consistency is derived in Proposition~\ref{prop:consistency}, while the asymptotic normality is provided in Proposition~\ref{prop:asymp_norm}.

\begin{proposition}[Consistency of $\estimator_{\nmatrices}$]
    \label{prop:consistency}
    The maximum likelihood estimator $\estimator_{\nmatrices}$ converges almost surely to the true parameter $\covcenter$ when $\nmatrices\to+\infty$.
\end{proposition}
\begin{proof}
    See supplementary materials.
    Notice that this proof relies on~\cite[Theorem 2]{huber1967behavior}.
\end{proof}

In order to be able to derive the asymptotic normality of the maximum likelihood estimator $\estimator_{\nmatrices}$, one further needs an additional assumption, which is provided in Assumption~\ref{assump:deriv_psi_bound}.
It states that the derivative of $\psi:\mathds{R}^{+}\to\mathds{R}$ needs to be bounded.
For the Wishart distribution, since $u(t)=1$, we have $\psi(t)=t$ and $\psi'(t)=1$.
Assumption~\ref{assump:deriv_psi_bound} is thus verified.
For the $t$-Wishart distribution, we have $u(t)=\frac{\dof+\nsamples\nfeatures}{\dof+t}$ (see Table~\ref{tab:EW_ex}).
Hence, $\psi(t)=\frac{\dof+\nsamples\nfeatures}{\dof+t}t$ and $\psi'(t)=\frac{\dof(\dof+\nsamples\nfeatures)}{(\dof+t)^2}$.
Since $\dof>0$, $\psi'(\cdot)$ is also bounded on $\mathds{R}^{+}$ in this case.

\begin{assumption}
    \label{assump:deriv_psi_bound}
    There exists $B>0$ such that the derivative of $\psi:\mathds{R}^+\to\mathds{R}$ defined in Section~\ref{subsec:mle:exist_unique} is upper bounded by $B$, \textit{i.e.}, for all $t\in\mathds{R}^+$, $\psi'(t)=u(t)+tu'(t)\leq B$.
\end{assumption}

\begin{proposition}[Asymptotic normality of $\estimator_{\nmatrices}$]
    \label{prop:asymp_norm}
    Under Assumption~\ref{assump:deriv_psi_bound}, one has
    \begin{equation*}
        \sqrt{\nmatrices}\vvec(\estimator_{\nmatrices}-\covcenter)
        \;
        \overset{(d)}{\underset{\nmatrices\to+\infty}{\longrightarrow}}
        \;
        \Normal{\MAT{\Delta}},
    \end{equation*}
    where $\MAT{\Delta}=\sigma_{1,\infty}(\MAT{I}_{\nfeatures^2}+\MAT{K}_{\nfeatures\nfeatures})(\covcenter\otimes\covcenter) + \sigma_{2,\infty}\vvec(\covcenter)\vvec(\covcenter)^\top$.
    We have $\sigma_{1,\infty}=\frac1{2\alpha}$ and $\sigma_{2,\infty}=-\frac{\beta}{\alpha(\alpha+\nfeatures\beta)}$, where $\alpha$ and $\beta$ are defined in Proposition~\ref{prop:fim}.
\end{proposition}
\begin{proof}
    See supplementary materials.
    Notice that this proof relies on the corollary of~\cite[Theorem 3]{huber1967behavior}.
\end{proof}

In the particular case of the Wishart distribution, $\sigma_{1,\infty}^{\textup{W}}=\frac1\nsamples$ and $\sigma_{2,\infty}^{\textup{W}}=0$.
Notice that this result can also be obtained by a direct application of the central limit theorem.
For the $t$-Wishart distribution, one has $\sigma_{1,\infty}^{t\textup{-W}}=\frac1\nsamples(1+\frac2{\dof+\nsamples\nfeatures})$ and $\sigma_{2,\infty}^{t\textup{-W}}=-\frac2\dof(1+\frac2{\dof+\nsamples\nfeatures})$.

When it comes to analyzing the performance of an estimator with respect to a statistical model, a very important object is the Cramér-Rao bound (CRB).
Indeed, it allows to provide a lower bound of the achievable error of an (optimal) estimator.
In the classical setting and its scalar form, the Cramér-Rao bound allows to provide a lower bound on the mean square error.
Given some estimator $\estimator_{\nmatrices}$ of a true parameter $\covcenter$, it is typically
\begin{equation*}
    \mathds{E}[\| \estimator_{\nmatrices} - \covcenter \|_2^2] \geq \frac1\nmatrices \tr(\MAT{F}_{\covcenter}^{-1}),
\end{equation*}
where $\MAT{F}_{\covcenter}$ is the Fisher information matrix (which can be obtained from the Fisher metric).
An estimator that actually reaches this Cramér-Rao bound is said to be Fisher-efficient.
When the parameter to estimate belongs to a manifold, a generalization of the Cramér-Rao bound is available: the so-called intrinsic Cramér-Rao bound~\cite{smith2005covariance,boumal2013intrinsic,breloy2018intrinsic}.
In this case, the error is no longer measured with the Euclidean distance but rather with a Riemannian distance $\delta(\cdot,\cdot)$ on the manifold.
The definition of the Fisher information matrix also changes a little and some curvature terms appear in the inequality.
The latter are however usually neglected.
The intrinsic Cramér-Rao bound then becomes
\begin{equation*}
    \mathds{E}[\delta^2(\estimator_{\nmatrices},\covcenter)] \geq \frac1\nmatrices \tr(\MAT{F}_{\covcenter}^{-1}).
\end{equation*}
In our case, $\delta^2(\cdot,\cdot)$ defined in~\eqref{eq:dist} corresponds to the Fisher distance of the statistical model we are interested in: Elliptical Wishart distributions.
In such a case, the Fisher information matrix is very simple: it is $\MAT{I}_d$, where $d$ is the dimension of the manifold, $\frac{\nfeatures(\nfeatures+1)}{2}$ for us.
Hence, the intrinsic Cramér-Rao bound of Elliptical Wishart distributions on $\SPDman$ for the Fisher distance~\eqref{eq:dist} is, given an estimator $\estimator_{\nmatrices}$ of the true parameter $\covcenter\in\SPDman$,
\begin{equation}
    \mathds{E}[\delta^2(\estimator_\nmatrices,\covcenter)] \geq \frac{\nfeatures(\nfeatures+1)}{2\nmatrices}.
    \label{eq:icrb}
\end{equation}
In the following proposition (Proposition~\ref{prop:fisher_efficient}), we aim to generalize the notion of Fisher efficiency to this intrinsic setting and show that the maximum likelihood estimator $\estimator_{\nmatrices}$ actually reaches the intrinsic Cramér-Rao bound.

\begin{proposition}[Intrinsic Fisher efficiency of $\estimator_{\nmatrices}$]
    \label{prop:fisher_efficient}
    The maximum likelihood estimator $\estimator_{\nmatrices}$ of the Elliptical Wishart distribution with true parameter $\covcenter\in\SPDman$ is such that
    \begin{equation*}
        \nmatrices \delta^2(\estimator_{\nmatrices},\covcenter)
        \;
        \overset{(d)}{\underset{\nmatrices\to+\infty}{\longrightarrow}}
        \;
        \chi^2 \left( \frac{\nfeatures(\nfeatures+1)}{2} \right).
    \end{equation*}
    Consequently,
    \begin{equation*}
        \nmatrices \mathds{E}[\delta^2(\estimator_{\nmatrices},\covcenter)]
        \;
        \underset{\nmatrices\to+\infty}{\longrightarrow}
        \;
        \frac{\nfeatures(\nfeatures+1)}{2}.
    \end{equation*}
\end{proposition}
\begin{proof}
    See supplementary materials.
\end{proof}

\section{Statistical learning with Elliptical Wishart distributions}
\label{sec:learning}

This section aims to provide practical learning applications for Elliptical Wishart distributions.
The maximum likelihood estimator of Elliptical Wishart distributions actually corresponds to a barycenter of a set of covariance matrices.
Barycenters are widely used in machine learning, in particular in the context of classification and clustering.
When dealing with covariance matrices, they have been for instance exploited in the nearest centroid classifier and in the $K$-means clustering algorithm; see, \textit{e.g.},~\cite{tuzel2008pedestrian,barachant2011multiclass,bouchard2024random}.
Hence, classification and clustering appear as natural applications of Elliptical Wishart distributions.
In Section~\ref{subsec:learning:classification}, a novel Bayesian classification method adapted to Elliptical Wishart distributions is derived.
It is obtained by leveraging the theory of discriminant analysis.
Notice that it generalizes the classifier of~\cite{ayadi2023t} from the $t$-Wishart distribution to all Elliptical Wishart distributions.
Then, in Section~\ref{subsec:learning:clustering}, a new clustering algorithm is designed by exploiting the proposed discriminant analysis algorithm of Section~\ref{subsec:learning:classification} and the $K$-means clustering algorithm.

\subsection{Discriminant analysis with Elliptical Wishart distributions}
\label{subsec:learning:classification}

Given $\nclasses$ classes, consider the training dataset $\mathcal{D}=\{(\cov_k,y_k)\}_{k=1}^\nmatrices$, where $\cov_k\in\SPDman$ and $y_k\in\llbracket1,\nclasses\rrbracket$ corresponds to the class label.
From $\mathcal{D}$, we aim at designing a classifier that, given an unlabeled sample $\cov\in\SPDman$, provides a class label $y\in\llbracket1,\nclasses\rrbracket$.
To obtain such classifier, we are inspired by linear and quadratic discriminant analysis.
Indeed, in the present work, we generalize discriminant analysis to the case of Elliptical Wishart random matrices.
To achieve this, for each class $z$, the first step is to compute the corresponding maximum likelihood estimator of the chosen Elliptical Wishart distribution, \textit{i.e.}, $\estimator^{(z)}$, the maximum likelihood estimator of the set $\{\cov_k\in\mathcal{D}:\, y_k=z\}_{k=1}^\nmatrices$.
The membership of an unlabeled $\cov\in\SPDman$ to a given class is then modeled by assuming that
\begin{equation}
    \cov | (y=z) \sim \EWishart{\covcenter^{(z)}}.
    \label{eq:assump_ewda}
\end{equation}
It yields the decision rule provided in Proposition~\ref{prop:ewda}.

\begin{proposition}[$\mathcal{EW}$DA's decision rule]
    \label{prop:ewda}
    The decision rule of the Elliptical Wishart discriminant analysis classifier, denoted $\mathcal{EW}$DA, is, for $\cov\in\SPDman$,
    \begin{equation*}
        y = \argmax_{z\in\{1,\dots,\nclasses\}} \quad \{ \delta^{\mathcal{EW}}_{(z)}(\cov) \}_{z=1}^\nclasses,
    \end{equation*}
    with the discriminant function
    \begin{equation*}
        \delta^{\mathcal{EW}}_{(z)}(\cov) = \log(\hat{\pi}^{(z)}) - \frac{\nsamples}{2}\log\det(\estimator^{(z)}) - \log \densitygenerator(\tr(\estimator^{(z)\,-1}\cov)),
    \end{equation*}
    where $\hat{\pi}^{(z)}$ is the proportion of the class $z$ in the training dataset $\mathcal{D}$.
\end{proposition}
\begin{proof}
    See supplementary materials.
\end{proof}

\begin{algorithm}[t]
    \caption{Elliptical Wishart distribution discriminant analysis $\mathcal{EW}$DA}
    \label{algo:ewda}
    \begin{algorithmic}
        \STATE \textbf{Training step}
        \vspace*{3pt}
        \hrule
        \vspace*{3pt}
        \STATE \textbf{Input:}
        degrees of freedom $n$,
        functions $\densitygenerator(\cdot)$ and $u(\cdot)$,
        training dataset $\mathcal{D}=\{(\cov_k,y_k)\}_{k=1}^\nmatrices$ in $\SPDman\times\llbracket1,\nclasses\rrbracket$.
        \\[5pt]
        \FOR{$z\in\llbracket1,\nclasses\rrbracket$}
            \STATE Compute MLE $\estimator^{(z)}$ of $\{\cov_k\in\mathcal{D}:\, y_k=z\}_{k=1}^\nmatrices$ with Algorithm~\ref{algo:fixed_point} or~\ref{algo:riem}
            \STATE Compute the proportion of class $z$ in $\mathcal{D}$: $\hat{\pi}^{(z)}=\frac1\nmatrices\sum_{k=1}^\nmatrices (y_k==z)$
        \ENDFOR
        \STATE \textbf{Return:} $\{ \estimator^{(z)}, \hat{\pi}^{(z)} \}_{z=1}^\nclasses$
        \vspace*{5pt}
        \hrule
        \vspace*{3pt}
        \STATE \textbf{Testing step}
        \vspace*{3pt}
        \hrule
        \vspace*{3pt}
        \STATE \textbf{Input:} unlabeled data $\cov\in\SPDman$.
        \FOR{$z\in\llbracket1,\nclasses\rrbracket$}
            \STATE Compute discriminant function $\delta^{\mathcal{EW}}_{(z)}(\cov)$ with Proposition~\ref{prop:ewda}
        \ENDFOR
        \STATE \textbf{Return:} label $y\in\llbracket1,\nclasses\rrbracket$ computed with Proposition~\ref{prop:ewda}
    \end{algorithmic}
\end{algorithm}

The Elliptical Wishart distribution discriminant analysis classifier $\mathcal{EW}$DA is summarized in Algorithm~\ref{algo:ewda}.
Concerning the Wishart distribution, it yields the so-called $\mathcal{W}$DA classifier.
Its discriminant function is given by
\begin{equation}
    \delta^{\mathcal{W}}_{(z)}(\cov) = \log(\hat{\pi}^{(z)}) - \frac{\nsamples}2\log\det(\estimator^{(z)}) - \frac12\tr(\estimator^{(z)\,-1}\cov).
\end{equation}
Furthermore, the $t$-Wishart distribution yields the so-called $t$-$\mathcal{W}$DA classifier.
Its discriminant function is
\begin{equation}
    \delta^{t\textup{-}\mathcal{W}}_{(z)}(\cov) = \log(\hat{\pi}^{(z)}) - \frac{\nsamples}2\log\det(\estimator^{(z)}) - \frac{\dof+\nsamples\nfeatures}2\log(1+\frac1\dof\tr(\estimator^{(z)\,-1}\cov)).
\end{equation}

\subsection{A $K$-means like clustering algorithm}
\label{subsec:learning:clustering}

The Elliptical Wishart discriminant analysis classifier $\mathcal{EW}$DA can be adapted to an unsupervised clustering scenario.
It is obtained by leveraging the $K$-means clustering method, which can be viewed as the adaptation of the nearest centroid classifier to the unsupervised clustering scenario.
Our resulting method is referred to as the Elliptical Wishart $K$-means clustering method.
In this setting, one has a set of $\nmatrices$ matrices $\{\cov_k\}_{k=1}^\nmatrices$ in $\SPDman$.
Given a certain number of classes $\nclasses$, the goal of clustering is to assign a label $y_k\in\llbracket1,\nmatrices\rrbracket$ to each $\cov_k$.
This is achieved with an iterative algorithm.
Each iteration $\ell$ consists of two steps.
The first one is the so-called assignment procedure, where, given $\nclasses$ barycenters $\{\estimator^{(z)}_{\ell}\}_{z=1}^\nclasses$, a label $y_{k,\ell}$ is assigned to each $\cov_k$ by exploiting the decision rule from Proposition~\ref{prop:ewda}.
The second step is a barycenter update one, where every maximum likelihood estimator $\estimator^{(z)}_{\ell}$ is recomputed from the dataset $\{\cov_k:\,y_{k,\ell}=z\}$.
This is repeated until some equilibrium is reached.
It is very well known that this type of clustering algorithms is very sensitive to the initialization of barycenters.
As proposed in~\cite{arthur2007k}, we consider employing several initializations and keep the results from the one maximizing the so-called inertia criterion
\begin{equation}
    \mathcal{I}(\{\cov_k,y_k\},\{\estimator^{(z)}\}_{z=1}^\nclasses) = \sum_{k=1}^\nmatrices \delta^{\mathcal{EW}}_{(y_k)}(\cov_k),
    \label{eq:clustering_inertia}
\end{equation}
where $\delta^{\mathcal{EW}}_{(z)}(\cov)$ is the discriminant function from Proposition~\ref{prop:ewda}.
The proposed clustering method is summarized in Algorithm~\ref{algo:ew_clustering}.

\begin{algorithm}
    \caption{Elliptical Wishart $K$-means clustering}
    \label{algo:ew_clustering}
    \begin{algorithmic}
        \STATE \textbf{Input:}
        degrees of freedom $n$,
        functions $\densitygenerator(\cdot)$ and $u(\cdot)$,
        dataset $\{\cov_k\}_{k=1}^\nmatrices$,
        number of classes $\nclasses$,
        tolerance $\alpha>0$,
        maximum iterations $\ell_{\textup{max}}$,
        number of initializations $M$.
        \\[5pt]
        \FOR{$m$ \textbf{in} $\llbracket1,M\rrbracket$}
            \STATE Randomly choose $\{k_z\}_{z=1}^\nclasses$ in $\llbracket1,\nmatrices\rrbracket$ and set $\estimator^{(z)}_{0}=\cov_{k_z}$
            \STATE Compute $\{y_{k,0}\}_{k=1}^\nmatrices$ with the decision rule in Proposition~\ref{prop:ewda}
            \STATE $\ell=0$
            \REPEAT
                \STATE $\ell=\ell+1$
                \STATE Compute $\{\estimator^{(z)}_{\ell}\}_{z=1}^\nclasses$ from $\{\cov_k:\,y_{k,\ell-1}=z\}$ with Algorithm~\ref{algo:fixed_point} or~\ref{algo:riem}
                \STATE Compute $\{y_{k,\ell}\}_{k=1}^\nmatrices$ with the decision rule in Proposition~\ref{prop:ewda}
            \UNTIL{$\frac1\nmatrices\sum_{k=1}^\nmatrices \| y_{k,\ell+1} - y_{k,\ell} \|_2<\alpha$ \textbf{or} $\ell>\ell_{\textup{max}}$}
            \STATE compute inertia $\mathcal{I}(m)$ for initialization $m$ with~\eqref{eq:clustering_inertia}
        \ENDFOR
        \STATE Compute $m_{\textup{max}} = \argmax_{m\in\llbracket1,M\rrbracket} \quad \{ \mathcal{I}(m) \}_{m=1}^M$
        \STATE \textbf{Return:} $\{y_k\}_{k=1}^\nmatrices$ associated with initialization $m_{\textup{max}}$
    \end{algorithmic}
\end{algorithm}

\section{Numerical experiments}
In this section, numerical experiments are conducted in order to investigate the interest of Elliptical Wishart distributions in practice.
In these experiments, we focus on the Wishart and $t$-Wishart distributions.
When computing the maximum likelihood of the $t$-Wishart distribution with the Riemannian algorithm (Algorithm~\ref{algo:riem}), we employ the Riemannian conjugate gradient implemented with pymanopt~\cite{boumal2014manopt,townsend2016pymanopt}.
First, in Section~\ref{subsec:sim_data}, we study the performance of the maximum likelihood estimator computed with Algorithms~\ref{algo:fixed_point} and~\ref{algo:riem} on simulated data.
More specifically, the two different algorithms are compared and the interest of the $t$-Wishart maximum likelihood estimator is evaluated as compared to the Wishart maximum likelihood estimator~\eqref{eq:mle_Wishart}.
Finally, in Section~\ref{subsec:real_data}, we explore the performance of our proposed Elliptical Wishart based statistical learning methods on real data.
Our classification method from Algorithm~\ref{algo:ewda} is employed on electroencephalographic (EEG) data.
Then, the clustering method provided in Algorithm~\ref{algo:ew_clustering} is applied to hyperspectral data.

\subsection{Simulated data}
\label{subsec:sim_data}

To evaluate the performance of the maximum likelihood estimator in practice, the first thing is to generate synthetic data.
We start by generating a center $\covcenter\in\SPDman$.
To do so, we set $\covcenter=\MAT{V}\MAT{\Lambda}\MAT{V}^\top$, where $\MAT{V}\in\mathds{R}^{\nfeatures\times\nfeatures}$ is uniformly distributed on the orthogonal group and $\MAT{\Lambda}\in\mathds{R}^{\nfeatures\times\nfeatures}$ is a positive definite diagonal matrix such that: the smallest and largest diagonal values are $1/\sqrt{c}$ and $\sqrt{c}$ ($c=10$ is the condition number with respect to inversion of $\MAT{\Lambda}$), and its other elements are uniformly distributed in between. 
Then, random sets $\{\cov_k\}_{k=1}^\nmatrices$ are drawn from the $t$-Wishart distribution $\tWishart{\covcenter}$.
These are obtained thanks to the procedure developed in~\cite{ayadi2024elliptical}.
For each combination of considered parameters ($\nfeatures$, $\nsamples$, $\nmatrices$, $\dof$), $200$ Monte Carlo repetitions are performed.
Furthermore, to measure the error of an estimator $\estimator$ of the true parameter $\covcenter$, we leverage the Fisher distance~\eqref{eq:dist}, \textit{i.e.}, $\mathrm{err}(\estimator)=\delta^2(\estimator,\covcenter)$.

\begin{figure}
    \centering
    \setlength\height{4.5cm} 
\setlength\width{0.53\linewidth}

\begin{tikzpicture}

\begin{axis}[
    width  =\width,
    height =\height,
    at     ={(0,0)},
    xlabel = {iterations},
    ylabel = {$\delta^2(\estimator,\covcenter)$~(dB)},
    title = {\small$\nsamples=100$},
    legend style={legend cell align=left,align=left,draw=none,fill=none,font=\footnotesize,legend columns=2,transpose legend}
    ]

    \foreach \x in {1,...,20}
    {
        \ifnum\x=1
            \addplot[color=myblue,line width=0.5pt] table [x=iter,y=error_\x,col sep=comma] {./figures/FP_n100_p10_df10_K300_MC20_cond10_iter5.txt};
            \addlegendentry{FP};
            \addplot[color=myorange,line width=0.5pt] table [x=iter,y=error_\x,col sep=comma] {./figures/RCG_n100_p10_df10_K300_MC20_cond10.txt};
            \addlegendentry{RCG};
        \else
            \addplot[forget plot,color=myblue,line width=0.5pt] table [x=iter,y=error_\x,col sep=comma] {./figures/FP_n100_p10_df10_K300_MC20_cond10_iter5.txt};
            \addplot[forget plot,color=myorange,line width=0.5pt] table [x=iter,y=error_\x,col sep=comma] {./figures/RCG_n100_p10_df10_K300_MC20_cond10.txt};
        \fi
    }
\end{axis}

\begin{axis}[
    width  =\width,
    height =\height,
    at     ={(0.83\width,0)},
    xlabel = {iterations},
    yticklabels  = {\empty},
    title = {\small$\nsamples=1000$},
    ]

    \foreach \x in {1,...,20}
    {
        \addplot[forget plot,color=myblue,line width=0.5pt] table [x=iter,y=error_\x,col sep=comma,skip coords between index={600}{5720}] {./figures/FP_n1000_p10_df10_K300_MC20_cond10_iter50.txt};
        \addplot[forget plot,color=myorange,line width=0.5pt] table [x=iter,y=error_\x,col sep=comma] {./figures/RCG_n1000_p10_df10_K300_MC20_cond10.txt};
    }
\end{axis}

\begin{axis}[
    width  =\width,
    height =\height,
    at     ={(0,-\height)},
    xlabel = {time (s)},
    ylabel = {$\delta^2(\estimator,\covcenter)$~(dB)},
    ]
    
    \foreach \x in {1,...,20}
    {
        \addplot[forget plot,color=myblue,line width=0.5pt] table [x=time_\x,y=error_\x,col sep=comma] {./figures/FP_n100_p10_df10_K300_MC20_cond10_iter5.txt};
        \addplot[forget plot,color=myorange,line width=0.5pt] table [x=time_\x,y=error_\x,col sep=comma] {./figures/RCG_n100_p10_df10_K300_MC20_cond10.txt};
    }
\end{axis}

\begin{axis}[
    width  =\width,
    height =\height,
    at     ={(0.83\width,-\height)},
    xlabel = {time (s)},
    yticklabels  = {\empty},
    ]

    \foreach \x in {1,...,20}
    {
        \addplot[forget plot,color=myblue,line width=0.5pt] table [x=time_\x,y=error_\x,col sep=comma,skip coords between index={600}{5720}] {./figures/FP_n1000_p10_df10_K300_MC20_cond10_iter50.txt};
        \addplot[forget plot,color=myorange,line width=0.5pt] table [x=time_\x,y=error_\x,col sep=comma] {./figures/RCG_n1000_p10_df10_K300_MC20_cond10.txt};
    }
\end{axis}

\end{tikzpicture}
    \vspace*{-40pt}
    \caption{Comparison over 20 runs between the fixed-point algorithm presented in Algorithm~\ref{algo:fixed_point} (FP) and the Riemannian optimization based algorithm provided in Algorithm~\ref{algo:riem} (RCG).
    For the Riemannian algorithm, a Riemannian conjugate gradient method is leveraged.
    The estimation error $\delta^2(\estimator,\covcenter)$ is plotted as a function of \textit{(i)} the number of iterations and \textit{(ii)} time.
    Fixed parameters are $\nfeatures=10$, $\nmatrices=300$ and $\dof=10$.}
    \label{fig:FP_vs_Riem}
\end{figure}
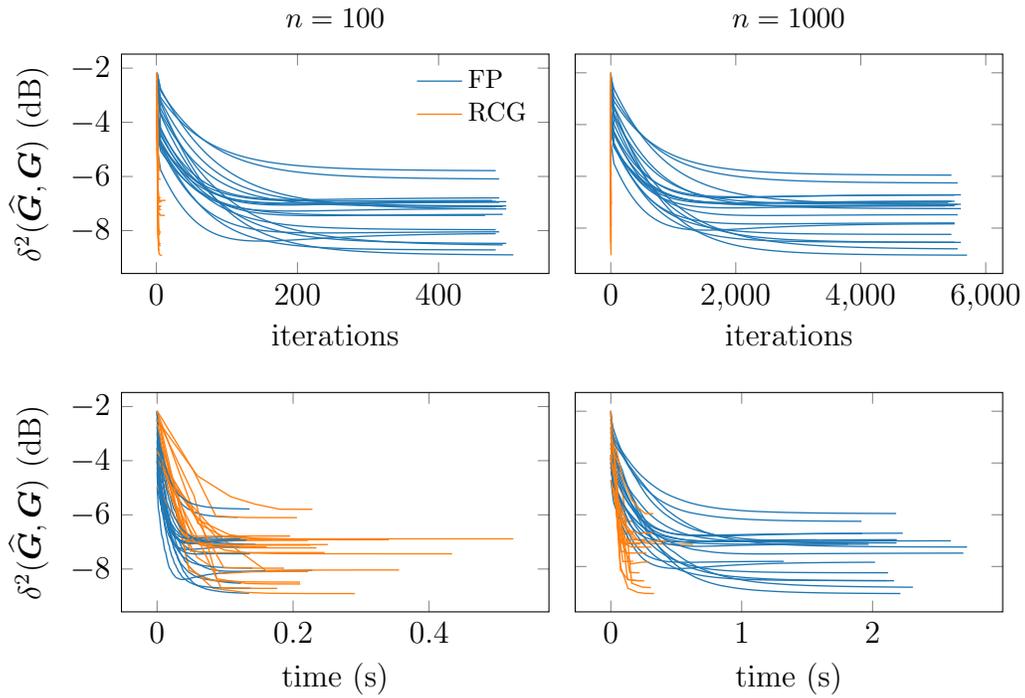

Our first task consists in comparing the fixed-point algorithm presented in Algorithm~\ref{algo:fixed_point} with the Riemannian optimization based algorithm provided in Algorithm~\ref{algo:riem}.
To do so, we observe the evolution of the error measure as a function of \textit{(i)} the number of iterations and \textit{(ii)} computation time.
Results are presented in Figure~\ref{fig:FP_vs_Riem}.
We observe that, in terms of iterations, in all considered cases, Algorithm~\ref{algo:riem}, denoted RCG, is drastically better than Algorithm~\ref{algo:fixed_point}, denoted FP.
Only a few iterations are required for RCG (around 10) while many are needed for FP (hundreds for $\nsamples=100$, thousands for $\nsamples=1000$).
As $\nsamples$ increases, FP takes more iterations to converge, while the number of iterations for RCG remains steady.
In terms of computation time, an iteration of FP appears significantly less costly than an iteration of RCG.
Consequently, when $\nsamples=100$, FP is slightly faster than RCG.
For $\nsamples=1000$, FP becomes quite slower than RCG.
Notice that while the implementation of FP is custom and probably close to be optimal, the one of RCG relies on a generic toolbox.
Hence, one can expect that it can be significantly improved in terms of computation time.
In conclusion, in the case of the $t$-Wishart distribution with such simulated data, RCG appears advantageous as compared to FP.
Interestingly, in the multivariate case, the fixed-point algorithm is usually the best solution, better than Riemannian optimization.

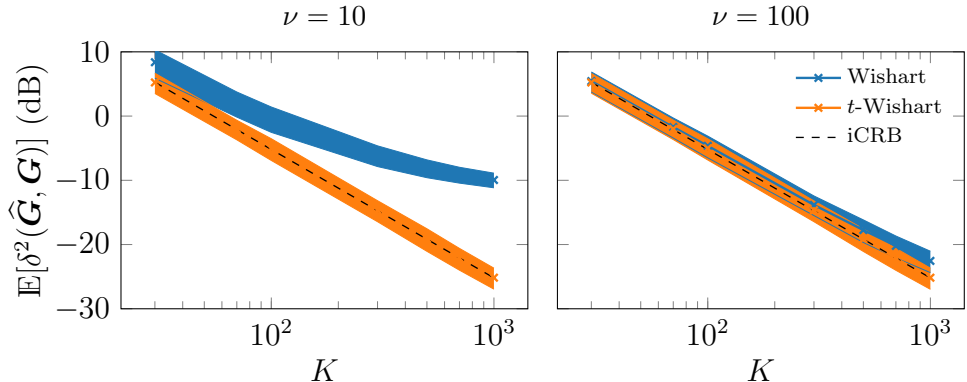
\begin{figure}
    \centering
    \setlength\height{5cm} 
\setlength\width{0.51\linewidth}

\begin{tikzpicture}

\begin{semilogxaxis}[
    width  =\width,
    height =\height,
    at     ={(0,0)},
    xlabel = {$\nmatrices$},
    ylabel = {$\mathds{E}[\delta^2(\estimator,\covcenter)]$~(dB)},
    ymax=10,
    ymin=-30,
    title = {\small$\dof=10$},
]

    \addplot[color=myblue,opacity=0.2,line width=0.1pt, forget plot, name path=A] table [x=k,y expr=20*log10(\thisrow{wishart_err_mean}+\thisrow{wishart_err_std}),col sep=comma] {./figures/estimation_n100_p10_df10_MC200_cond10.txt};
    \addplot[color=myblue,opacity=0.2,line width=0.1pt, forget plot, name path=B] table [x=k,y expr=20*log10(\thisrow{wishart_err_mean}-\thisrow{wishart_err_std}),col sep=comma] {./figures/estimation_n100_p10_df10_MC200_cond10.txt};
    \addplot[color=myblue,opacity=0.2, forget plot] fill between[of=A and B];

    \addplot[color=myorange,opacity=0.2,line width=0.1pt, forget plot, name path=A] table [x=k,y expr=20*log10(\thisrow{twishart_err_mean}+\thisrow{twishart_err_std}),col sep=comma] {./figures/estimation_n100_p10_df10_MC200_cond10.txt};
    \addplot[color=myorange,opacity=0.2,line width=0.1pt, forget plot, name path=B] table [x=k,y expr=20*log10(\thisrow{twishart_err_mean}-\thisrow{twishart_err_std}),col sep=comma] {./figures/estimation_n100_p10_df10_MC200_cond10.txt};
    \addplot[color=myorange,opacity=0.2, forget plot] fill between[of=A and B];

    \addplot[color=myblue,line width=0.9pt,mark=x,mark size=2pt] table [x=k,y expr=20*log10(\thisrow{wishart_err_mean}),col sep=comma] {./figures/estimation_n100_p10_df10_MC200_cond10.txt};

    \addplot[color=myorange,line width=0.9pt,mark=x,mark size=2pt] table [x=k,y expr=20*log10(\thisrow{twishart_err_mean}),col sep=comma] {./figures/estimation_n100_p10_df10_MC200_cond10.txt};

    \addplot[color=black,dashed,line width=0.5pt] table [x=k,y expr=20*log10(\thisrow{crb}),col sep=comma] {./figures/estimation_n100_p10_df10_MC200_cond10.txt};

\end{semilogxaxis}

\begin{semilogxaxis}[
    width  =\width,
    height =\height,
    at     ={(0.83\width,0)},
    xlabel = {$\nmatrices$},
    %
    yticklabel = {\empty},
    ymax=10,
    ymin=-30,
    title = {\small$\dof=100$},
    legend style={legend cell align=left,align=left,draw=none,fill=none,font=\scriptsize,legend columns=1,transpose legend}
]

    \addplot[color=myblue,opacity=0.2,line width=0.1pt, forget plot, name path=A] table [x=k,y expr=20*log10(\thisrow{wishart_err_mean}+\thisrow{wishart_err_std}),col sep=comma] {./figures/estimation_n100_p10_df100_MC200_cond10.txt};
    \addplot[color=myblue,opacity=0.2,line width=0.1pt, forget plot, name path=B] table [x=k,y expr=20*log10(\thisrow{wishart_err_mean}-\thisrow{wishart_err_std}),col sep=comma] {./figures/estimation_n100_p10_df100_MC200_cond10.txt};
    \addplot[color=myblue,opacity=0.2, forget plot] fill between[of=A and B];

    \addplot[color=myorange,opacity=0.2,line width=0.1pt, forget plot, name path=A] table [x=k,y expr=20*log10(\thisrow{twishart_err_mean}+\thisrow{twishart_err_std}),col sep=comma] {./figures/estimation_n100_p10_df100_MC200_cond10.txt};
    \addplot[color=myorange,opacity=0.2,line width=0.1pt, forget plot, name path=B] table [x=k,y expr=20*log10(\thisrow{twishart_err_mean}-\thisrow{twishart_err_std}),col sep=comma] {./figures/estimation_n100_p10_df100_MC200_cond10.txt};
    \addplot[color=myorange,opacity=0.2, forget plot] fill between[of=A and B];
    
    \addplot[color=myblue,line width=0.9pt,mark=x,mark size=2pt] table [x=k,y expr=20*log10(\thisrow{wishart_err_mean}),col sep=comma] {./figures/estimation_n100_p10_df100_MC200_cond10.txt};
    \addlegendentry{Wishart};

    \addplot[color=myorange,line width=0.9pt,mark=x,mark size=2pt] table [x=k,y expr=20*log10(\thisrow{twishart_err_mean}),col sep=comma] {./figures/estimation_n100_p10_df100_MC200_cond10.txt};
    \addlegendentry{$t$-Wishart};

    \addplot[color=black,dashed,line width=0.5pt] table [x=k,y expr=20*log10(\thisrow{crb}),col sep=comma] {./figures/estimation_n100_p10_df100_MC200_cond10.txt};
    \addlegendentry{iCRB};

\end{semilogxaxis}

\end{tikzpicture}
    \vspace*{-40pt}
    \caption{Mean and standard deviation of error measure $\delta^2(\estimator,\covcenter)$ as a function of the number of matrices $\nmatrices$ for the Wishart maximum likelihood estimator~\eqref{eq:mle_Wishart} and the $t$-Wishart maximum likelihood estimator computed with Algorithm~\ref{algo:riem} (Riemannian conjugate gradient).
    The intrinsic Cramér-Rao bound~\eqref{eq:icrb} is also displayed.
    Fixed parameters are $\nfeatures=10$ and $\nsamples=100$.
    Degrees of freedom for simulated $t$-Wishart random matrices $\{\cov_k\}_{k=1}^\nmatrices$ are $\dof=10$ (left) and $\dof=100$ (right).
    In both cases, the $t$-Wishart maximum likelihood estimator is computed with the correct value of $\dof$.
    Means and standard deviations are computed over $200$ Monte Carlo repetitions.}
    \label{fig:perf_vs_samples}
\end{figure}

Secondly, we investigate the practical interest of the $t$-Wishart maximum likelihood estimator as compared to the simpler Wishart maximum likelihood estimator~\eqref{eq:mle_Wishart}.
To do so, we look at the performance of both estimators on simulated data for various values of degrees of freedom $\dof$ of the $t$-distribution, and various number of random matrices $\nmatrices$.
Obtained results are displayed in Figure~\ref{fig:perf_vs_samples}.
We observe that for $\dof=10$, when random matrices are far from the Wishart distribution, the accuracy of the Wishart maximum likelihood estimator is far from the one of the $t$-Wishart maximum likelihood estimator, especially as $\nmatrices$ grows.
When $\dof=100$, random matrices get closer to the Wishart distribution and the performance of the Wishart maximum likelihood estimator get closer to the one of the $t$-Wishart maximum likelihood estimator.
It is still a little inferior, especially for large values of $\nmatrices$.
In summary, as expected%
\footnote{
    Since it perfectly corresponds to the simulated data.
},
the $t$-Wishart maximum likelihood estimator is the best in both cases.
Furthermore, the $t$-Wishart maximum likelihood estimator actually reaches the intrinsic Cramér-Rao bound~\eqref{eq:icrb}.
This tends to show that Proposition~\ref{prop:fisher_efficient} of intrinsic Fisher efficiency is verified in practice.

In conclusion, these experiments on simulated data validate the actual performance of the $t$-Wishart maximum likelihood estimator computed with Algorithms~\ref{algo:fixed_point} and~\ref{algo:riem}.
This estimator can now be exploited on real data statistical learning applications.

\subsection{Statistical learning on real data}
\label{subsec:real_data}

The first real data statistical learning experiment that is considered consists in classifying electroencephalographic (EEG) data.
In the second one, we deal with the clustering of hyperspectral data.

\subsubsection{Classification of EEG data}
In order to perform EEG classification, we exploit MOABB~\cite{jayaram2018moabb,chevallier2024largest}, which provides numerous datasets and benchmarking tools for fair comparisons.
In this work, we consider two different brain computer interface (BCI) paradigms: motor imagery (MI) and steady states visually evoked potentials (SSVEP).
MI is a mental process where the subject mentally visualize a physical action such as moving their left or right hand, feet, tongue, \textit{etc}.
Each action corresponds to a class and the goal is to identify them from raw signals.
SSVEP signals are natural responses to repetitive visual stimuli at specific frequencies, \textit{e.g.}, blinking leds.
The visual cortex synchronizes with the stimuli and the resulting sinusoidal signal can be recorded.
In this setting, the different classes correspond to the different considered frequencies for the stimuli -- and usually also a resting state class (no stimulus).

Different datasets available on MOABB are employed.
For MI, we consider BNCI2014-001 (9 subjects, 2 sessions, $\nfeatures=16$, $\nsamples=1000$, $\nmatrices=288$, $\nclasses=4$); BNCI2014-004 (9 subjects, 5 sessions, $\nfeatures=3$, $\nsamples=1125$, $\nmatrices=360$, $\nclasses=2$); BNCI2015-001 (12 subjects, 2 or 3 sessions, $\nfeatures=13$, $\nsamples=1280$, $\nmatrices=200$, $\nclasses=2$); and Weibo2014 (10 subjects, 1 session, $\nfeatures=10$, $\nsamples=800$, $\nmatrices=80$, $\nclasses=7$).
For SSVEP, we use Kalunga2016 (12 subjects, 1 session, $\nfeatures=8$, $\nsamples=1280$, $\nmatrices=64$, $\nclasses=4$).
For further details on these datasets, the reader is referred to the documentation of MOABB.

Constructing the SPD samples $\{\cov_k\}_{k=1}^\nmatrices$ in $\SPDman$ from the raw EEG signals $\{\dataMat_k\}_{k=1}^\nmatrices$ in $\mathds{R}^{\nfeatures_{\dataMat}\times\nsamples}$ depends on the paradigm.
For MI, we simply take $\cov_k=\dataMat_k\dataMat_k^\top$.
For SSVEP, it is more complicated.
For each stimulus frequency $f_z$, the trial $\dataMat_k$ is filtered with a bandpass filter around the frequency of interest $f_z$, yielding $\dataMat^{(f_z)}_k$.
Resulting filtered trials are then stacked, \textit{i.e.}, $\widetilde{\dataMat}_k = ([\dataMat^{(f_z)}_k]_{z=1}^{\nclasses})^\top$.
Finally, $\cov_k=\widetilde{\dataMat}_k\widetilde{\dataMat}_k^\top$.
Notice that in this case, $\nfeatures=\nclasses\nfeatures_{\dataMat}$, meaning that the SPD matrices can quickly become very big.

In addition to comparing the classification accuracies of the Wishart and $t$-Wishart discriminant analysis classifiers $\mathcal{W}$DA and $t$-$\mathcal{W}$DA derived from Algorithm~\ref{algo:ewda} in Section~\ref{subsec:learning:classification}, we consider the so-called Riemannian MDM classifier~\cite{barachant2011multiclass}.
This classifier relies on the Fisher distance of the multivariate Normal distribution, which corresponds to~\eqref{eq:dist} with $\alpha=1$ and $\beta=0$.
Each class is characterized by the Fréchet mean of corresponding training sample covariance matrices.
The decision rule consists in selecting the class whose Fréchet mean is the closest to the unknown covariance.

Obtained classification accuracies are given in Table~\ref{tab:EEG_classif}.
We observe that for every dataset, the overall classification accuracies contain quite a lot of variability.
Hence, no thorough interpretation can be taken out of these.
We can notice that our proposed classifier $t$-$\mathcal{W}$DA seems competitive with the state-of-the-art classifier MDM and even appears slightly better (especially for Weibo2014).
$\mathcal{W}$DA is also competitive but feels slightly inferior (especially for BNCI2015-001).
This tends to show the practical interest of $t$-$\mathcal{W}$DA in the context of EEG classification.

\begin{table}[t!]
    \centering
    \renewcommand{\arraystretch}{1.4}
    \small
    \begin{tabular}{clccc}
        \hline
        \normalsize paradigm & \normalsize dataset & \normalsize MDM & \normalsize $\mathcal{W}$DA & \normalsize $t$-$\mathcal{W}$DA 
        \\
        \hline
        \multirow{4}{*}{\normalsize MI}
        & BNCI2014-001 & $65.1\pm15.6$ & $63.7\pm15.5$ & $\mathbf{67.1\pm16.5}$
        \\
        & BNCI2014-004 & $77.8\pm15.7$ & $76.3\pm15.7$ & $\mathbf{78.8\pm14.8}$
        \\
        & BNCI2015-001 & $\mathbf{80.5\pm13.7}$ & $70.7\pm15.7$ & $79.4\pm14.1$
        \\
        & Weibo2014 & $34.4\pm6.4$ & $41.8\pm9.4$ & $\mathbf{42.0\pm9.0}$
        \\
        \hline
        \normalsize SSVEP & Kalunga2016 & $79.1\pm10.0$ & $81.79\pm13.44$ & $\mathbf{82.5\pm10.3}$
        \\
        \hline
    \end{tabular}
    \caption{Classification accuracies for MDM~\cite{barachant2011multiclass}, $\mathcal{W}$DA and $t$-$\mathcal{W}$DA classifiers from Algorithm~\ref{algo:ewda} over various EEG datasets.
    All datasets are available in MOABB~\cite{jayaram2018moabb,chevallier2024largest}.}
    \label{tab:EEG_classif}
\end{table}

\subsubsection{Clustering of hyperspectral data}

We now consider the clustering of hyperspectral remote sensing datasets.
In this paper, we consider the datasets Salinas, KSC and Indian Pines%
\footnote{
    These three datasets are available at \url{https://www.ehu.eus/ccwintco/index.php/Hyperspectral_Remote_Sensing_Scenes}.
}.
Each dataset is composed of one hyperspectral image of the ground somewhere on earth containing a unique number of reflectance bands.
These datasets also feature a unique number of classes and possess annotated ground truths.
Some areas are labeled as ``undefined'' and considered unreliable.
It is thus usual to exclude these from the accuracy measures.
Nevertheless, they are included in the clustering procedure to ensure a realistic evaluation.

Following previous works~\cite{collas2021probabilistic,bouchard2024random}, the preprocessing of each dataset consists in three main steps.
The first one is to normalize the data by subtracting the image global mean.
Then, principal component analysis (PCA) is applied in order to select the number of features $\nfeatures$.
For the three datasets, we set $\nfeatures=5$.
Finally, a sliding window with overlap is used around each pixel for data sampling.
This determines $\nsamples$.
In all cases, we choose $\nsamples=25$ in the present work.

As for EEG data, three different methods are considered: the Wishart and $t$-Wishart $k$-means clustering derived from Algorithm~\ref{algo:ew_clustering}, and the $K$-means on $\SPDman$ which is the counterpart of the MDM classifier.
In the three cases, for Salinas and KSC, 5 different initializations are used while 10 different initializations are exploited for Indian Pines.
To measure accuracy, the first step is to apply a linear optimization assignment algorithm in order to align the clustered images with the ground truths.
Then, two usual performance measures are used: the averaged overall accuracy (acc.), as well as the averaged mean intersection over union (mIoU).

Obtained results are reported in Table~\ref{tab:hyper_cluster}.
Again these experiments are too limited to actually draw meaningful conclusions.
We still observe that our proposed clustering method based on $t$-$\mathcal{W}$DA yields the best results for the three considered scenarios.
On the other hand, the clustering method based on $\mathcal{W}$DA performs quite poorly as compared to other considered methods on these datasets with the chosen settings.

\begin{table}[t!]
    \centering
    \renewcommand{\arraystretch}{1.4}
    \begin{tabular}{ccccccc}
         \hline
         \multirow{2}{*}{dataset} & \multicolumn{2}{c}{MDM} & \multicolumn{2}{c}{Wishart} & \multicolumn{2}{c}{$t$-Wishart}
         \\[-7pt]
         & acc. & mIoU & acc. & mIoU & acc. & mIoU
         \\
         \hline
         Salinas & 52.7 & 31.7 & 43.5 & 25.1 & \textbf{60.3} & \textbf{42.2}
         \\
         KSC & 26.3 & 16.9 & 9.7 & 12.9 & \textbf{28.0} & \textbf{19.0}
         \\
         Indian Pines & 37.6 & 26.6 & 37.3 & 23.9 & \textbf{41.3} & \textbf{26.1}
         \\
         \hline
    \end{tabular}
    \caption{Averaged overall accuracies (acc.) and averaged mean intersection over union (mIoU) for the clustering of three hyperspectral datasets with the Wishart and $t$-Wishart $k$-means clustering derived from Algorithm~\ref{algo:ew_clustering}, and the $K$-means on $\SPDman$ which is the counterpart of the MDM classifier.}
    \label{tab:hyper_cluster}
\end{table}

\bibliographystyle{elsarticle-num}
\bibliography{biblio}

\end{document}